\documentclass{article}

\usepackage{microtype}
\usepackage{graphicx}
\usepackage{subfigure}
\usepackage{booktabs} 

\usepackage{latexsym}
\usepackage{tabularx} 
\usepackage{array}
\usepackage{multirow}
\usepackage{hyperref}

\usepackage[accepted]{icml2025}

\usepackage{amsmath}
\usepackage{amssymb}
\usepackage{mathtools}
\usepackage{amsthm}

\usepackage[capitalize,noabbrev]{cleveref}

\theoremstyle{plain}
\newtheorem{theorem}{Theorem}[section]

\newtheorem{lemma}[theorem]{Lemma}
\newtheorem{corollary}[theorem]{Corollary}
\theoremstyle{definition}

\theoremstyle{remark}

\usepackage{amsmath}
\usepackage{amsthm}
\usepackage{amssymb}

\usepackage[textsize=tiny]{todonotes}

\begin{document}

\twocolumn[
\icmltitle{ThoughtProbe: Classifier-Guided Thought Space Exploration \\ Leveraging LLMs' Intrinsic Reasoning}



\icmlsetsymbol{equal}{*}

\begin{icmlauthorlist}
\icmlauthor{Zijian Wang}{yyy}
\icmlauthor{Chang Xu}{yyy}
\end{icmlauthorlist}

\icmlaffiliation{yyy}{School of Computer Science, The University of Sydney, NSW, Australia}

\icmlcorrespondingauthor{Zijian Wang}{zwan0998@uni.sydney.edu.au}
\icmlcorrespondingauthor{Chang Xu}{c.xu@sydney.edu.au}

\vskip 0.3in
]




\makeatletter\renewcommand{\Notice@String}{}\makeatother

\printAffiliationsAndNotice{} 


\begin{abstract}
Pre-trained large language models (LLMs) have been demonstrated to possess intrinsic reasoning capabilities that can emerge naturally when expanding the response space. 
However, the neural representation mechanisms underlying these intrinsic capabilities and approaches for their optimal utilization remain inadequately understood.
In this work, we make the key discovery that a simple linear classifier can effectively detect intrinsic reasoning capabilities in LLMs' activation space, particularly within specific representation types and network layers.
Based on this finding, we propose a classifier-guided search framework that strategically explore a tree-structured response space.
In each node expansion, the classifier serves as a scoring and ranking mechanism that efficiently allocates computational resources by identifying and prioritizing more thoughtful reasoning directions for continuation.
After completing the tree expansion, we collect answers from all branches to form a candidate answer pool. 
We propose a branch-aggregation selection method that marginalizes over all supporting branches by aggregating their thoughtfulness scores, thereby identifying the optimal answer from the pool.
Experimental results show that our framework's comprehensive exploration not only covers valid reasoning chains but also effectively identifies them, achieving significant improvements across multiple arithmetic reasoning benchmarks.


\end{abstract}

\section{Introduction}
\label{Introduction}

Pre-trained large language models (LLMs) often struggle to demonstrate robust reasoning capabilities, primarily due to their tendency to provide short intuitive responses rather than engaging in a deliberative reasoning process \cite{rae2021scaling}.
This limitation makes them particularly prone to errors when tackling tasks that demand multi-step logical reasoning, such as mathematical arithmetic \cite{yao2024tree, wei2022chain}.

Some works aim to enhance pre-trained LLMs' reasoning through approaches that require human involvement - whether designing prompts or tuning with annotated data \cite{zelikman2022star, hoffman2024training, zhang2024chain, wei2022chain, kojima2022large}. 
However, these methods passively inject human reasoning prior knowledge into LLMs rather than actively exploring and leveraging their intrinsic reasoning capabilities.

\begin{figure*}[ht]
\begin{center}
\centerline{\includegraphics[scale = 0.5]{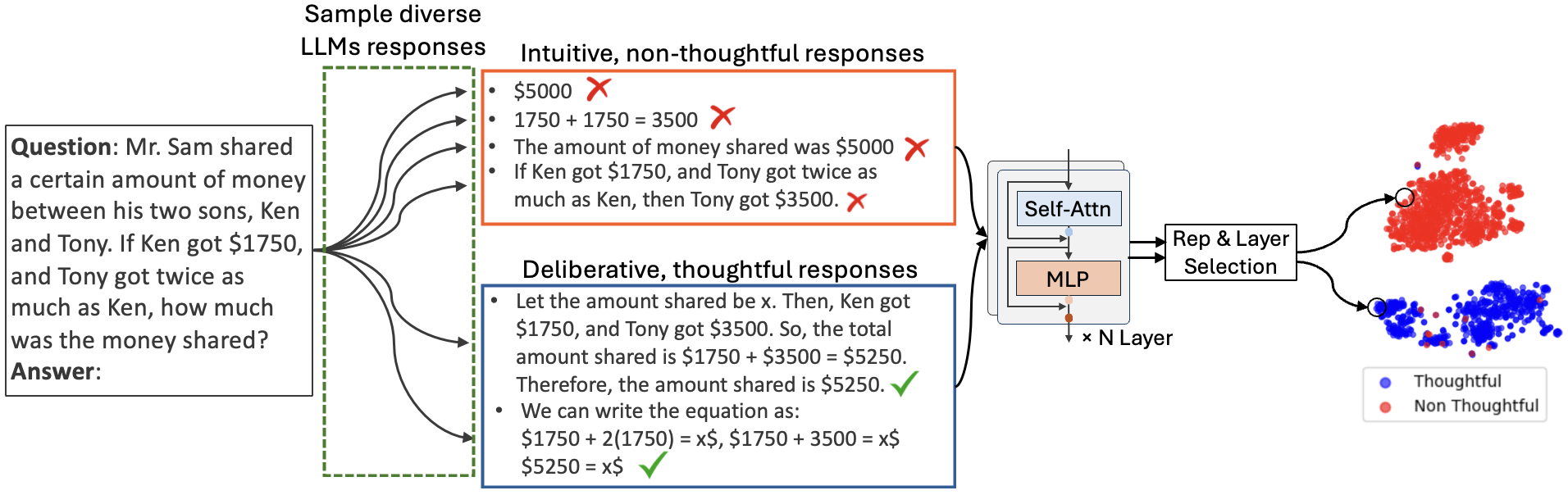}}
\caption{Pre-trained LLMs could naturally generate both thoughtful and non-thoughtful responses when sampling multiple times, with these responses being linearly separable in the activation space.}
\label{Intro}
\end{center}
\vspace{-2mm}
\end{figure*}

Recent evidence suggests that pre-trained LLMs may possess intrinsic reasoning capabilities even without explicit human intervention. 
When sampling multiple responses with only a question as prompt, lengthy reasoning chains can emerge naturally and could lead to correct answers \cite{wang2024chain, wang2022self}, as shown in Figure \ref{Intro}, indicating intrinsic reasoning abilities that single-pass inference fails to capture.
Building on this observation, recent frameworks have successfully employed search algorithms to explore response space and identify effective reasoning chains \cite{wu2024inference, luo2024improve, snell2024scaling, lightman2023let, wang2024math}, using various guiding signals from semantic analysis \cite{zhang2024thought} to LLM-based verbal evaluation \cite{xie2024self}.

However, key questions remain unexplored: How are reasoning capabilities represented within LLMs' internal activations, and how can these intrinsic abilities be effectively harnessed? 
Better understanding of these aspects could lead to more interpretable and reliable methods.

To address this issue, our research aim to investigate pre-trained LLMs intrinsic reasoning capabilities at hidden representation level.
We first introduce the concept of thoughtfulness - the degree to which a sampled response demonstrates deliberative reasoning content rather than short intuitive answers - as a proxy to study and leverage LLMs' intrinsic reasoning capabilities.
Our empirical evidence reveals that the Linear Representation Hypothesis (LRH) \cite{zou2310representation, von2024language} extends to thoughtfulness, showing that thoughtful and non-thoughtful responses are linearly separable in LLMs' activation space, as shown in Figure \ref{Intro}. 
Specifically, we find that this separability is most prominent in certain representation types and layer ranges, where a simple linear classifier can effectively distinguish thoughtful responses from non-thoughtful ones.
Also, we show the classifier consistently assigns higher scores to correct, thoughtful responses and lower scores to incorrect ones, regardless of response length.

Motivated by this discovery, we present ThoughtProbe, a classifier-guided inference-time computing framework that systematically leverages LLMs' intrinsic thoughtfulness for effective response space exploration. 
Our framework is built on the intuition that thoughtful responses are more likely to contain correct reasoning steps, thus increasing the probability of finding valid solution paths within the search space.
The exploration process is structured as an iterative tree search, with the initial question as the root node. 
At each step, multiple token sequences are generated as candidate child nodes, whose thoughtfulness score is evaluated by our classifier through probing hidden representations. 
By selecting more thoughtful candidates for continuation, we efficiently allocate computational resources and increase the likelihood of including correct reasoning trajectories in our exploration tree.
This exploration process continues until either reaching the termination token or exhausting the computational budget.

Upon completion of tree expansion, we obtain multiple branches, each leading to a candidate answer, forming a comprehensive answer pool. 
To determine the optimal one, we propose a branch-aggregation selection through value marginalization that considers thoughtfulness score across all branches leading to each candidate answer.
Specifically, the value of each answer is computed by aggregating the thoughtfulness score of all its supporting branches, with the final answer selected as the one that achieves the highest marginal value.

Experiments on multiple mathematical reasoning benchmarks demonstrate that our method consistently outperforms existing approaches, achieving significant improvements. 
Our work provides new insights into enhancing LLMs' reasoning capabilities and opens up promising directions for developing more robust reasoning systems.


\section{Preliminary}

\subsection{LLMs Architecture and Hidden Representation}

To provide a foundation for the discussion, we first describe the basic structure of a Transformer-based LLM architecture \cite{vaswani2017attention}. 
The input text is initially tokenized into a sequence of tokens, which are then mapped to embeddings to form the initial representation sequence $ \mathbf{x}^{(0)} \in \mathbb{R}^{T \times d_\text{emb}}$.
Here, $T$ is the sequence length, and $d_\text{emb}$ is the embedding dimension.

The embeddings are then processed through multiple Transformer layers.
In each layer $l$, its hidden states are composed of three components: activations from multi-head self-attention (MHA), multi-layer perceptron (MLP), and residual connections. This process can be formulated as:

\vskip -0.2in
\begin{align*}
\mathbf{a}^{(l)}_\text{attn} &= \text{MHA}(\mathbf{h}^{(l)}) & \text{(Attention activations)} \\
\mathbf{a}^{(l)}_\text{mlp} &= \text{MLP}(\mathbf{a}^{(l)}_\text{attn} + \mathbf{h}^{(l)})  & \text{(MLP activations)} \\
\mathbf{h}^{(l+1)} &= \mathbf{a}^{(l)}_\text{mlp} + \mathbf{a}^{(l)}_\text{attn} + \mathbf{h}^{(l)}  & \text{(Hidden states)}
\end{align*}

\subsection{LLMs Reasoning Structure}
Reasoning structures typically manifest in two fundamental topologies: sequential chains and branching trees. 
The chain structure reflects the step-by-step nature of logical deduction, while the tree structure captures the exploration of multiple potential reasoning paths. 
Below, we formally define these structures and their probabilistic formulations.

\textbf{Reasoning Chain} For an input question $Q$, a reasoning chain is defined as a sequence of intermediate thought steps $R = [Q, r_1, r_2, ..., r_N]$, leading to a final answer $A$. 
Here, $r_i$ represents an intermediate thought at the $i$-th step, and $N$ denotes the chain length. 
The answer can be extracted by appending a trigger prompt at the end of the chain, like ``Therefore, the answer is".
The probability of generating such a chain can be formalized as:
\begin{equation*}
   P(R, A|Q) = P(r_1|Q) \prod_{i=2}^N P(r_i|Q, r_{1:i-1}) \cdot P(A|Q, R)
\end{equation*}
where $P(r_1|Q)$ is the probability of the first step, $P(r_i|Q, r_{1:i-1})$ is the probability of the $i$-th step, and $P(A|Q, R)$ is the probability of the final answer. 
At each step $i$, a new thought $r_i$ is appended to form $R = [Q, r_1, ..., r_{i-1}, r_i]$.

More specifically, each reasoning step $r_i$ is itself a token sequence, which can be further decomposed as:

\begin{equation*}
P(r_i|Q, r_{1:i-1}) = \prod_{t=1}^{T_i} P(r_i^t|Q, r_{1:i-1}, r_i^{1:t-1})
\end{equation*}
where, $r_i^t$ denotes the $t$-th token in the $i$-th reasoning step, $T_i$ represents the total number of tokens in the $i$-th step, $r_i^{1:t-1}$ represents the previously generated tokens in the current step. 
In each token generation, the hidden representation $Rep(r_i^t)$ of token $r_i^t$ is accessible for probing.

\textbf{Branching Chains into Trees} 
By sampling diverse tokens at each reasoning step, a single chain can branch into a tree structure, where $Q$ serves as the root node, each node represents an intermediate reasoning step.
This tree-based expansion explores multiple reasoning branches simultaneously and can increase the probability of covering the correct reasoning chain and answer.

At each step $r_i$, we could sample $k$ different continuations:

\begin{equation*}
\{r_i^1, r_i^2, ..., r_i^k\} \sim P_k(r_i \mid Q, r_{1:i-1})
\end{equation*}

Here, $r_i^j$ represents the $j$-th sampled continuation at step $i$. 
Each root-to-leaf chain forms a distinct branch, leading to its answer, and collectively these branches generate an answer pool $\mathcal{A} = \{A_1, A_2, ..., A_p\}$.

While the tree structure improves solution coverage, it introduces two key challenges: 
(1) Candidate Selection: How to evaluate and prioritize promising children nodes in each exploration step? 
(2) Answer Determination: How to select the optimal answer from the pool $\mathcal{A}$?

\section{ThoughtProbe: Classifier-guided Reasoning Tree Exploration}
This section presents our ThoughtProbe framework that guide the response space exploration by probing LLMs internal thoughtfulness.
We first validate this linear separability through comprehensive probing experiments across different LLMs.
Building on these findings, we then introduce a classifier-guided beam search algorithm that systematically explores the response space to construct a diverse answer pool. 
Finally, we propose marginalization methods to aggregate these answers based on their thoughtfulness evaluation, enabling effective optimal answer selection.

\subsection{Linear Probing in Thoughtfulness}
\label{main:Linear Probing in Thoughtfulness}
In this section, we conduct linear probing experiments to analyze how thoughtfulness is encoded across different architectural components (hidden states, attention outputs, and MLP activations) of LLMs. 
Through examining the classification performance on thought/non-thought responses, we reveal distinct patterns in how different LLMs represent thoughtful reasoning.

\textbf{Training Data Construction}
To create a dataset for thoughtfulness classification, we collect paired thought/non-thought responses using questions from the GSM8K \cite{cobbe2021training} training set. 
For each question, we sample 10 distinct responses, each limited to 240 tokens.
We establish a binary labeling criterion: responses are labeled as thoughtful if they provide correct answers and contain more than 30 tokens, while responses that are both incorrect and shorter than 30 tokens are labeled as non-thoughtful (intuitive). 
This labeling scheme propagates to all response tokens and their corresponding hidden activations within each response, creating our training dataset for thoughtfulness classification at hidden representation level. 
More details are in the appendix  \ref{app: classifier training data}.

\textbf{Linear Classifier}
We employ Logistic Regression (LR) as our linear classifier. 
LR models the probability of thoughtfulness through a two-step process: first computing the logit (log-odds) using a linear function $\mathbf{w}^\top \mathbf{x} + b$, then transforming it to probability of positive via the sigmoid function $\sigma$.
\vspace{-2mm}
\begin{align*}
\text{logit} = \ln\frac{P(y=1|\mathbf{x})}{P(y=0|\mathbf{x})} &= \mathbf{w}^\top \mathbf{x} + b \\
P(y=1 | \mathbf{x}) = \sigma(\text{logit}) &= \frac{1}{1 + e^{-(\mathbf{w}^\top \mathbf{x} + b)}}
\end{align*}

where $\mathbf{w}$ is the weight vector, $b$ is the bias term, and $\mathbf{x}$ is the input feature vector. 
LR fits this linear decision boundary by minimizing the binary cross-entropy loss.

\textbf{Setup}
Using our collected 1,000 thought/non-thought pairs (split 80\% for training and 20\% for testing), we investigate the linear separability of thoughtfulness in model activations from three LLMs: Mistral-7b \cite{jiang2023mistral}, Gemma-2-2b \cite{team2024gemma}, phi-1.5 \cite{li2023textbooks}. 
For each layer and representation type (Hidden states, Attention activations, and MLP activations), we train LR classifiers and evaluate their performance using AUC-ROC, and F1-score.

\vspace{-3mm}
\begin{figure}[ht]
\begin{center}
\centerline{\includegraphics[scale = 0.40]{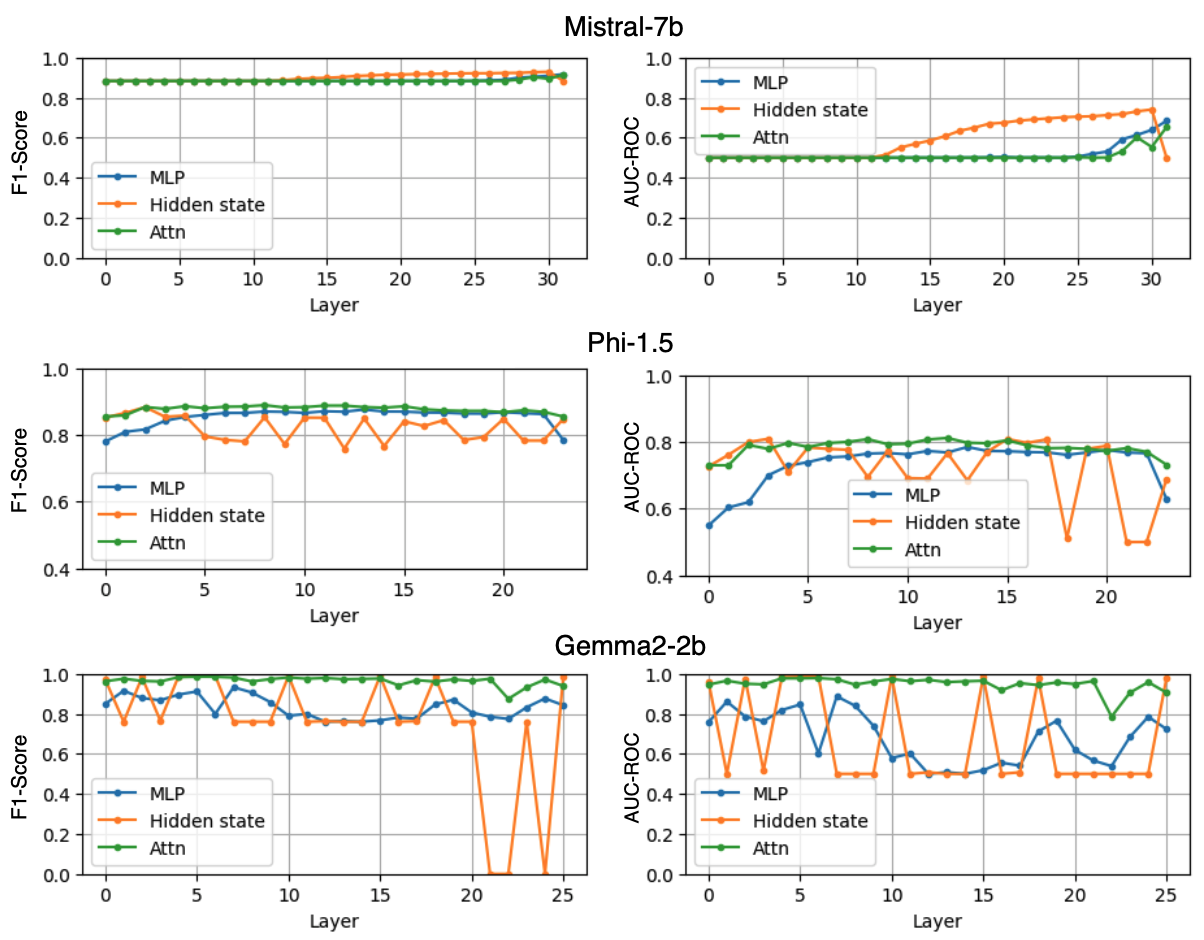}}
\caption{Classification performance (F1-Score and AUC-ROC) of linear classifiers across different representation types and LLMs.}
\label{classifier}
\end{center}
\vskip -0.2in
\end{figure}

\textbf{Results} Figure \ref{classifier} shows the plot of classification performance. Our analysis reveals two key findings:
(1) Representation type analysis: 
Different LLMs demonstrate distinct patterns in expressing thoughtfulness across representation types. 
In Mistral-7b, while MLP and Attention activations perform similarly, hidden states achieve the best performance.
In Phi-1.5, hidden states fluctuate considerably, but Attention outputs remain stable and outperform MLP.
In Gemma2-2b, hidden states and MLP activations show significant fluctuation, while Attention outputs maintain superior performance.
These diverse patterns suggest that LLMs may adopt different architectural strategies for encoding thoughtful reasoning capabilities.
(2) Layer-wise analysis: Layer depth affects classification performance differently across LLMs.
Mistral-7b exhibits an increasing trend from shallow to deep layers, indicating a progressive refinement of thoughtfulness features.
In contrast, Phi-1.5 and Gemma2-2b show fluctuation patterns across layers with no clear upward or downward trends, suggesting a more distributed encoding of thoughtfulness concept.

\textbf{Conclusion} Despite variations across representations and layers, all models achieve over 80\% performance with their optimal configurations, indicating that thoughtfulness are linearly separable in these neural representations.


\subsection{Classifier-guided Beam Search}
We propose a classifier-guided beam search where classifier's logit serve as reward signals \cite{sun2024rethinking} for ranking and selecting thoughtful candidates. 
The theoretical foundation is established by the following theorem:

\begin{theorem}[Logit-Reward Order Preservation]
Let $l(x)$ be the logit value of a binary classifier trained on preference data derived from Bradley-Terry model \cite{bradley1952rank}, where preference pairs are treated as binary classification data. 
Let $r(x)$ be the reward function in the original Bradley-Terry model. The following preference ordering equivalence holds:
\begin{equation*}
r(x_1) > r(x_2) \iff l(x_1) > l(x_2)
\end{equation*}
(See appendix for proof)
\end{theorem}

This theorem establishes that the classifier's logit outputs can serve as a theoretically sound proxy for thoughtfulness evaluation and candidate ranking in our beam search process. 
We also empirically validate its ranking ability in Figure \ref{fig:score}. 
In the left subplot, thoughtful correct responses consistently achieve higher logit values than short intuitive incorrect ones. 
Remarkably, the right subplot shows that thoughtful correct responses also maintain higher logit values compared to lengthy incorrect responses, suggesting that our classifier captures response quality independent of length.
Detailed analyses are presented in the appendix \ref{app:score analysis}.

\begin{figure}[h]
\begin{center}
\centerline{\includegraphics[scale = 0.24]{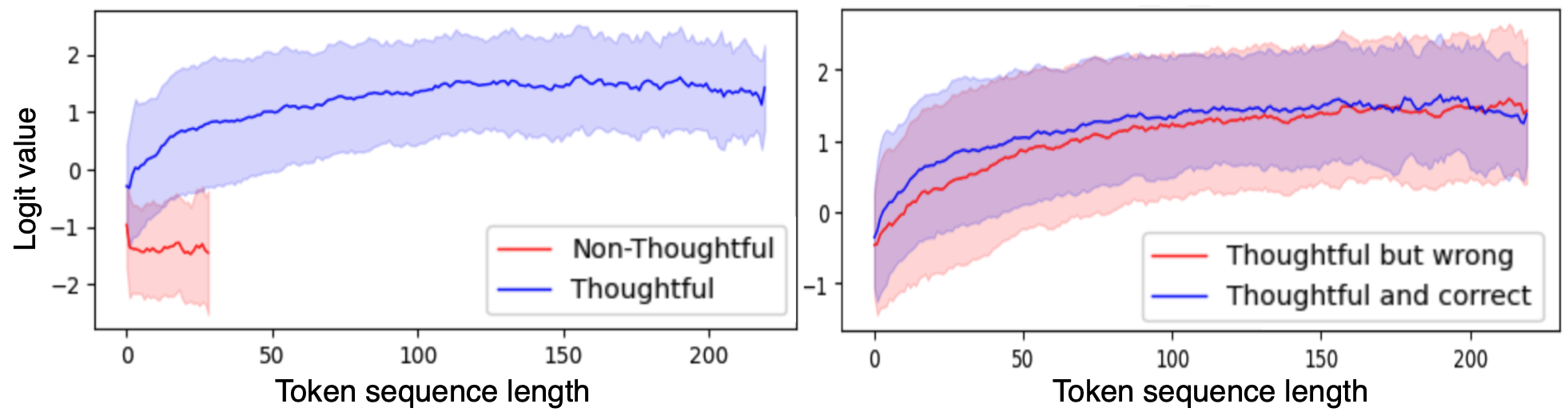}}
\caption{Mean logit values and variance regions along the token sequence in Phi-1.5. Left: Comparison between lengthy thoughtful correct responses and concise incorrect intuitive responses. Right: Comparison between lengthy thoughtful correct responses and lengthy incorrect responses.}
\label{fig:score}
\end{center}
\vspace{-3mm}
\end{figure}

\begin{figure*}[ht]
\begin{center}
\centerline{\includegraphics[scale = 0.35]{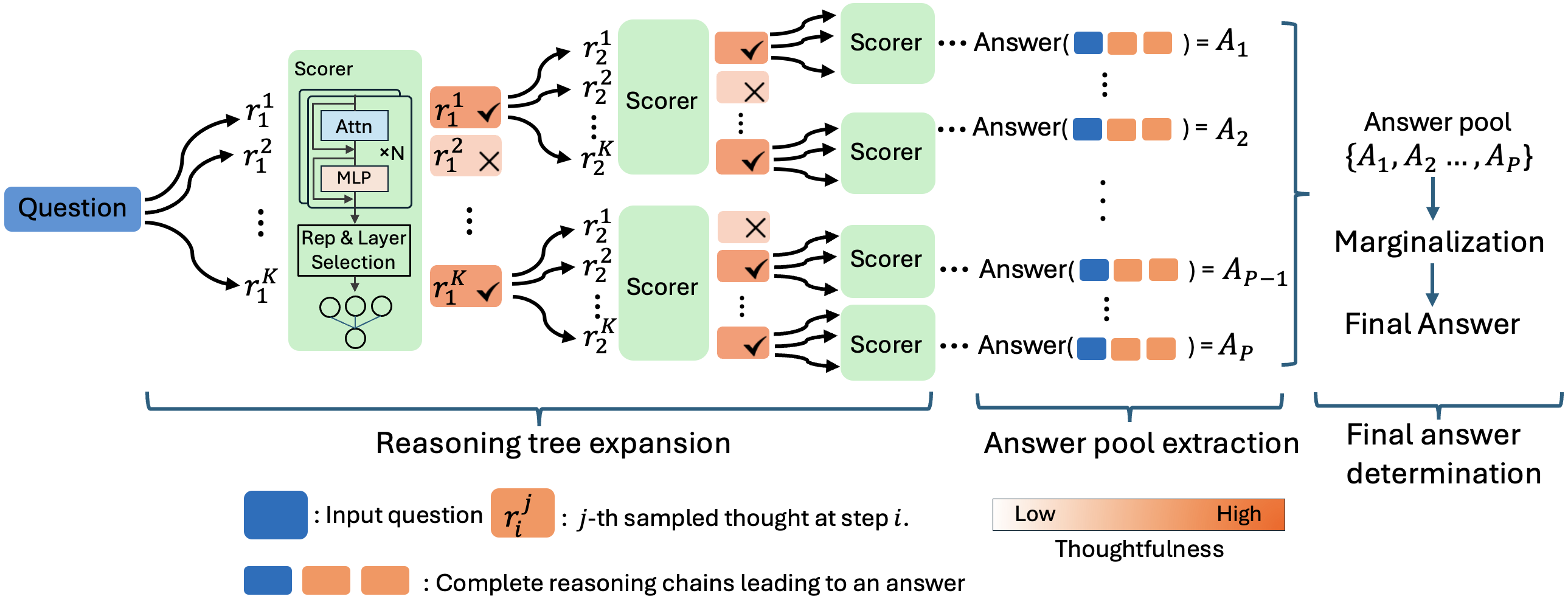}}
\caption{Our classifier-guided tree exploration framework. 
At each parent node, multiple candidates are sampled and evaluated by a pre-trained classifier in activation space. 
Nodes are selected for further expansion based on thoughtfulness scores. Each exploration branch produces a candidate answer, forming an answer pool from which the final answer is determined through marginalization across all branches.}
\vspace{-3mm}
\label{framework}
\end{center}

\end{figure*}

Specifically, for a parent node (root question or intermediate reasoning step), the tree expansion process is formulated as follows:

\textbf{Diverse Beam Construction}: 
The process begins by generating diverse candidate continuations, organized into a beam. 
To encourage diversity, stochasticity must be introduced during token sequence generation.
In this paper, we employ Top-K-Start Greedy Decoding, which explores alternative top-$k$ tokens at the first decoding step, followed by greedy decoding for subsequent steps \cite{wang2024chain}.
The resulting $k$ reasoning chains, denoted as $B = \{R_1, R_2, ..., R_k\}$, represent potential continuations with associated hidden states, forming the initial beam for further processing.

\textbf{Thoughtfulness Evaluation via Classifier}: 
Once the beam is constructed, a pre-trained classifier is used to evaluate the thoughtfulness of each candidate . 
The classifier operates on the hidden state representations of the chains and assigns a score to each one. 
Specifically, for a candidate chain $R_i$, the thoughtfulness score $S_i$ is computed as $S_i = l(Rep(R_i[-1]))$, where $l(\cdot)$ is the logit output of classifier and $Rep(R_i[-1])$ represents the hidden state of the last token in $R_i$. 

\textbf{Beam Pruning by Score Ranking}: 
After scoring, all candidate chains are ranked based on their thoughtfulness scores, and only the top-$n$ highest-scoring candidates are retained for further expansion.
The pruned beam, denoted as $B'$, is defined as $B' = {R_{\sigma(i)} \mid i \leq n}$, where $\sigma$ is the permutation that sorts the scores in descending order, and $R_{\sigma(i)}$ represents the candidate corresponding to the $i$-th highest score. 
By dynamically adjusting the beam width $n$, we can control the trade-off between exploration breadth and computational efficiency. 
This pruning step ensures that only promising reasoning paths are preserved, effectively reducing computational overhead while maintaining the quality of the reasoning process.

\textbf{Implementation Details}: 
Our framework consists of two phases: a branching phase for systematic exploration with depth $m$ and beam width $n$, followed by a completion phase for final generation.
During the branching phase, we iteratively expand the tree for $m$ steps. 
At each step $i$, we first generate $k$ candidate responses for each node and select the top-$n$ candidates based on their thoughtfulness scores, with each candidate expanded by generating a sequence of $T_i$ tokens.
In the completion phase, all leaf nodes from the branching phase are extended using greedy decoding until either reaching a completion token or the maximum length limit.
For input formatting, we adopt a simple question-answer template: ``Question:[question]\verb|\n|Answer:" without employing any additional prompting techniques.

\subsection{Answer Pool Marginalization}
After completing the tree expansion process, we generate final answers by appending the prompt ``Therefore, the answer is'' to each branch, resulting in an answer pool $\mathcal{A} = \{A_1, A_2, ..., A_p\}$. 
To select the final answer from the pool, several straightforward approaches can be applied: (1) majority voting based on answer frequency, and (2) single-branch selection that selects the answer from individual branch with the highest score metrics (\textit{e.g.}, final score or mean score).
Instead, we propose branch-aggregation selection that determines the final answer by aggregating branch score metrics for each answer.

Specifically, for each candidate answer $A_i$, we collect its supporting branches $R(A_i)$, which consists of all branches that arrive at $A_i$ as their final answer, formally defined as $R(A_i) = \{R \mid \text{answer}(R) = A_i\}$.
Then we compute the value of each branch from its node score sequence $[S_1, S_2, ..., S_N]$, using its final score $S_N$ as the branch value.
For each unique answer, we then aggregate the values of all its supporting branches by summation: $Value(A_i) = \sum_{R \in R(A_i)} Value(R)$.
Finally, we select the answer with the highest aggregated value as our final answer: $A^* = argmax_{A_i \in \mathcal{A}} Value(A_i)$.
We provide a detailed comparative analysis of different answer selection methods in Section \ref{sec: Answer Value Calculation Analysis}.

\section{Experiments}

\textbf{Dataset and LLMs} We evaluate our method on popular mathematical benchmarks: 
(1) the GSM8K benchmark of math word problems \cite{cobbe2021training}; 
(2) MultiArith \cite{roy2016solving}; 
(3) the SVAMP dataset of math word problems with varying structures  \cite{patel2021nlp};
(4) the MAWPS benchmark \cite{koncel2016mawps}.
We use same LLMs in Section \ref{main:Linear Probing in Thoughtfulness}, which are Mistral-7b, Gemma2-2b and phi-1.5. 

\textbf{Baselines}
We compare our approach with four representative baselines:
(1) Greedy Decoding: A straightforward approach that selects the highest-probability token at each step of generation, simulating fast, intuitive responses.
(2) Zero-shot Chain-of-Thought (Zs CoT) \cite{kojima2022large}: Enhances reasoning by appending the input "Let's think step by step" to questions, encouraging step-wise problem-solving without task-specific training.
(3) Chain-of-Thought Decoding (CoT-Dec) \cite{wang2024chain}: Generates multiple solution paths and selects the most confident one based on the average probability margin between the top two token predictions in the answer segment.
(4) Self-consistency (SC) \cite{wang2022self}: Employs a majority voting mechanism across multiple generated responses to identify the most consistent answer.

\textbf{Hyperparameters Config}
For the branching phase, we set the depth $m$ = 3 and beam width $n$ = 3. At each step, we generate $k$ = 10 candidates and select top-$n$ based on thoughtfulness scores, with token generation lengths $T_i$ = [1, 20, 20] for steps $i$ = 1,2,3. For the completion phase, we extend each leaf node with two steps of greedy decoding, generating 100 tokens per step.
A detailed analysis of how depth $m$ and beam width $n$ affect the framework's performance is presented in Section \ref{sec: Search Space Analysis}.

\begin{table}[h]
\resizebox{0.5\textwidth}{!}{
\begin{tabular}{cccccc}
\hline
Models                      & Methods & GSM8K       & MultiArith       & SVAMP             & MAWPS \\ \hline
\multirow{5}{*}{Mistral-7b} & Greedy  & 11.92       & 15.16            & 52.66             & 58.29 \\
                            & SC      & 17.13       & 27.22            & 58.00             & 66.56 \\
                            & Zs CoT  & 26.17       & 50.47            & 56.33             & 69.81 \\
                            & CoT-Dec & 25.79       & 39.76            & 58.66             & 64.78 \\
                            & Ours    &\textbf{40.18}  &\textbf{58.57} &\textbf{61.33} &\textbf{80.64} \\ \hline
                            
\multirow{5}{*}{Gemma2-2b}  & Greedy  & 6.42           & 5.53           & 38.53           & 46.16 \\
                            & SC      & 7.59           & 8.41           & 40.00        & 47.00 \\
                            & Zs CoT  & 16.92          & 42.11          & 39.33           & 51.69 \\
                            & CoT-Dec & 14.34          & 33.22          & 38.99           & 50.28 \\
                            & Ours    & \textbf{20.62} & \textbf{50.00} & \textbf{48.66}          & \textbf{63.86} \\ \hline

\multirow{5}{*}{Phi-1.5}    & Greedy  & 5.69             & 24.44          & 24.33            & 33.74 \\
                            & SC      & 25.02            & 33.88          & 29.03            & 39.16 \\
                            & Zs CoT  & 7.21             & \textbf{83.88} & 39.33            & 65.18 \\
                            & CoT-Dec & 23.12            & 25.00          & 23.66            & 50.05 \\
                            & Ours    & \textbf{37.38}   & 80.56          & \textbf{45.66}   &\textbf{68.45} \\ \hline
\end{tabular}
}
\vskip -0.1in
\caption{Problem solving accuracy compared with baselines across LLMs and datasets}
\label{tab: main results}
\end{table}

\subsection{Main Experimental Analysis}
As shown in Table \ref{tab: main results}, our method consistently outperforms baseline approaches in most scenarios, achieving substantial improvements in problem solving accuracy. 

\textbf{Cross-Model Analysis}
Our method shows robust performance gains across different LLM scales. 
For the larger Mistral-7b model, we observe the most significant improvements, with our method achieving 40.18\% accuracy on GSM8K, surpassing the strongest baseline (Zs CoT) by 14.01\%.
The performance advantage maintains for smaller models like Gemma2-2b and Phi-1.5, where our method improves GSM8K accuracy by 3.7\% and 12.36\% respectively compared to their best baselines. 
This demonstrates our method's effectiveness and generalizability across different model scales.

\textbf{Cross-Dataset Analysis}
Our method shows varying effectiveness across different datasets.
On GSM8K's complex multi-step problems, we demonstrate consistent superiority across all models.
For MultiArith, while achieving strong performance with Mistral-7b (58.57\%) and Gemma-2-2b (50.00\%), Phi-1.5 shows slightly lower accuracy (80.56\%) compared to Zs CoT (83.88\%), suggesting simpler arithmetic problems might benefit less from our approach.
On SVAMP and MAWPS, we maintain consistent improvements, with notable gains on MAWPS (3.27\%-12.17\% over the best baseline).
Notably, we train our classifier only on GSM8K training set and use this single classifier across all datasets, demonstrating strong generalization to various mathematical reasoning datasets.

\begin{table}[]
\label{answer selection}
\resizebox{0.5\textwidth}{!}{
\begin{tabular}{cccccc}

\hline
LLMs & Methods & GSM8K & MultiArith & SVAMP &WAMPS\\ \hline
\multirow{5}{*}{\rotatebox{90}{Mistral-7b}} 
& Cover Rate  & 85.44          & 91.65          & 90.33          & 94.33 \\
& F Agg/Sgl    & 40.18/27.84     & \textbf{58.57}/32.78  & \textbf{61.33}/52.45  & \textbf{80.64}/63.18  \\
& M Agg/Sgl     & 38.21/24.92        & 55.15/33.42        & 58.44/51.52       & 77.33/61.42\\
& IR Agg/Sgl   & \textbf{42.92}/ 23.52 & 57.53/35.63       & 60.21/47.42         & 79.33/64.21\\
& Vote & 39.21 & 56.15 & 59.44 & 78.33 \\ \hline
\multirow{5}{*}{\rotatebox{90}{Gemma2-2b}}
& Cover Rate   & 79.65                & 84.33                 & 88.44                    & 90.74 \\
& F Agg/Sgl    & \textbf{20.62}/11.52 & 50.00/25.53         & \textbf{48.66}/16.82   & \textbf{63.86}/35.42 \\
& M Agg/Sgl    & 18.15/10.83         & 47.77/27.63         & 45.33/23.63             & 61.33/43.85\\
& IR Agg/Sgl   & 21.53/13.53        & \textbf{51.21}/34.42 & 47.44/19.42            & 62.33/40.91\\
& Vote          & 19.21               & 48.15                & 46.44                     & 62.33            \\ \hline

\multirow{5}{*}{\rotatebox{90}{Phi-1.5}}
& Cover Rate      & 84.33                 & 89.42                     & 88.63                 & 92.82 \\
& F Agg/Sgl       & 37.38/21.72         & \textbf{80.56}/56.86    &\textbf{45.66}/29.74    & 68.45/49.72 \\
& M Agg/Sgl       & 35.77/20.44          & 77.21/49.63            & 42.33/31.42         & 65.53/50.82\\
& IR Agg/Sgl      & \textbf{38.21}/21.93 & 79.53/48.82                     & 44.65/30.84         & \textbf{69.84}/51.72\\
& Vote            & 36.21               & 78.15                       &  43.44                & 66.49 \\ \hline
\end{tabular}
}
\vskip -0.1in
\caption{Performance comparison of different answer selection methods. F Agg/Sgl, M Agg/Sgl, and IR Agg/Sgl represent the accuracy of branch-aggregation/single-branch selection using final scores, mean scores, and increase ratio respectively. Vote shows the accuracy of majority voting baseline.}
\vspace{-5mm}
\label{tab: answer value}
\end{table}

\begin{table*}[h]
\centering
\resizebox{0.90\textwidth}{!}{
\begin{tabular}{cccccccc}
\hline
\multirow{2}{*}{LLMs}   & \multirow{2}{*}{Dataset}  & \multicolumn{3}{c}{LR} 
                        & \multicolumn{3}{c}{SVM} \\ \cline{3-8}   
                        &   & MLP  & Attn       & \multicolumn{1}{c|}{Hidden states}   & MLP & Attn  & Hidden states \\ \hline
\multirow{4}{*}{Mistral-7b}    & GSM8K         & 35.21/18.42  & 34.57/17.23           & \textbf{40.18}/13.75  & 36.43/17.42 & 33.37/6.23 & 38.32/12.39\\
                               & MultiArith    & 51.55/16.33 & 49.82/21.65          & \textbf{58.57}/23.39  & 49.55/18.33 & 50.59/18.65 & 57.45/20.11   \\
                               & SVAMP         & 35.43/28.72 & 34.65/23.23          & \textbf{61.33}/25.76   & 36.43/27.42 & 35.43/26.42 & 60.39/17.23    \\
                               & WAMPS         & 69.55/39.33 & 72.81/42.65          & \textbf{80.64}/52.46  & 68.55/38.92 & 77.52/7.33 & 79.38/18.65    \\ \hline
\multirow{4}{*}{Gemma2-2b}     & GSM8K         & 15.41/6.51  & \textbf{20.62}/15.86 & 17.41/12.39     & 18.41/5.62 & 18.41/4.81 & 19.91/9.71   \\
                               & MultiArith    & 42.41/17.74  & \textbf{50.00}/21.23 & 40.92/29.72    & 46.41/6.23 & 48.41/10.63 & 41.65/15.85   \\
                               & SVAMP         & 35.43/13.92 & \textbf{48.66}/27.23  & 45.33/30.76   & 36.43/16.42 & 47.43/5.42 & 44.81/13.84    \\
                               & WAMPS         & 48.55/28.74 & 63.86/29.65  & 47.64/23.59            & 49.55/27.33 & \textbf{64.55}/26.37 & 46.84/16.65    \\ \hline
\multirow{4}{*}{Phi-1.5}       & GSM8K         & 15.41/8.69  & \textbf{37.38}/11.82 & 14.14/9.28    & 16.41/4.84 & 36.41/23.48 & 13.83/7.93    \\
\multicolumn{1}{l}{}           & MultiArith    & 46.41/24.82 & \textbf{80.56}/21.23 & 45.41/24.61    & 47.41/15.71 & 75.41/34.72 & 44.28/19.47   \\
\multicolumn{1}{l}{}           & SVAMP         & 34.43/16.71 & 45.66/27.75  & 43.39/25.76            & 35.43/21.42 & \textbf{46.43}/24.24 & 43.14/24.85    \\
\multicolumn{1}{l}{}           & WAMPS         & 47.55/27.48 & \textbf{68.45}/49.65  & 46.64/32.53   & 48.55/26.33 & 67.55/45.62 & 45.72/29.38     \\ \hline
\end{tabular}
}
\vskip -0.1in
\caption{Performance comparison of different classifiers (LR and SVM) and representations (MLP, Attention, Hidden states) using accuracy scores on top-3 and bottom-3 layers (reported as top-3/bottom-3)}
\label{tab: classifier feature analysis}
\end{table*}

\subsection{Answer Marginalization Analysis}
\label{sec: Answer Value Calculation Analysis}

Table \ref{tab: answer value} presents a comprehensive comparison of different approaches for final answer selection from the answer pool.
We first examine the coverage rate - the percentage of correct answers present in the pool - which indicates an upper bound for selection accuracy.
The high coverage rates (79\%-94\%) demonstrate that our exploration strategy effectively traverses the response space and captures valid reasoning chains.

We also evaluate two main selection paradigms: single-branch selection and branch-aggregation selection, across three score sequence metrics: final scores, average scores, and increase ratio (defined as the proportion of score improvements between adjacent nodes).
Our analysis shows that branch-aggregation selection consistently outperforms single-branch selection across all metrics. Among the three metrics, final scores yield the best performance, followed by increase ratio, while average scores show relatively inferior results.
Additionally, we benchmark these methods against the baseline majority voting approach, which shows superior performance to single-branch selection but falls short of branch-aggregation selection.


\subsection{Classifier Feature Analysis}
Table \ref{tab: classifier feature analysis} shows the performance comparison of different classifiers features, including classifier type, representation type and layers range.

\textbf{Classifier type Study}
We comparing Support Vector Machine(SVM) and LR classifiers, we observe their comparable performance across different representations, layers, and LLMs.
While SVM shows slightly better results in some cases, the differences are marginal, suggesting both classifiers can effectively guide the search process.

\textbf{Representation Layer Analysis}
We analyze the impact of layer by comparing top-3 and bottom-3 layers based on their classification F1-scores.
The results show that across all LLMs, using top-performing layers consistently outperforms bottom layers.
For GSM8K, the average improvements are 31.36\%, 29.79\%, and 30.04\% on Mistral-7b, Gemma-2-2b, and Phi-1.5 respectively, demonstrating that layer selection significantly affects search effectiveness.

\textbf{Representation Type Study}
Hidden states yield the best search performance for Mistral-7b, while attention activations prove more effective for both Gemma-2-2b and Phi-1.5. 
This pattern mirrors the relative strengths we observed in classification performance, suggesting a consistent relationship between classifier logit and reward.


\subsection{Search Space Size Analysis}

\label{sec: Search Space Analysis}
We investigate how different search space configurations affect model performance by varying beam width $n$ and tree depth $m$.
For each configuration, we maintain the initial sampling size $k=10$ while adjusting width $n \in \{1,2,3,4,5,6\}$ and depth $m \in \{1, 2, 3, 4, 5, 6\}$. 
All generated chains are constrained to a maximum length of 240 tokens, with tokens evenly distributed across depth steps ($T_i = 240/m$ tokens per step).
Figure \ref{fig: Search Space} demonstrates how performance varies with different combinations of width and depth, using Phi-1.5 on GSM8K.

\textbf{Beam Width Impact}
The accuracy improves substantially as the beam width increases, demonstrating the benefits of maintaining more parallel branches at each expansion step. 
The improvement trend begins to plateau around width = 4, suggesting that maintaining 3-4 parallel reasoning trajectories provides sufficient exploration while remaining computationally efficient. 
Further increasing the beam width yields diminishing returns, possibly due to the introduction of less more noise than thoughtful content.

\textbf{Search Depth Study}
The accuracy improves as search depth increases and reaches its peak at depth 3 or 4.
Beyond this optimal depth, performance gradually declines, suggesting that deeper searches may accumulate errors and explore irrelevant reasoning paths. 
This optimal depth aligns with human solving patterns, as most mathematical problems can be effectively solved within 3-4 key reasoning steps.


\begin{figure}[ht]
\begin{center}
\centerline{\includegraphics[scale = 0.34]{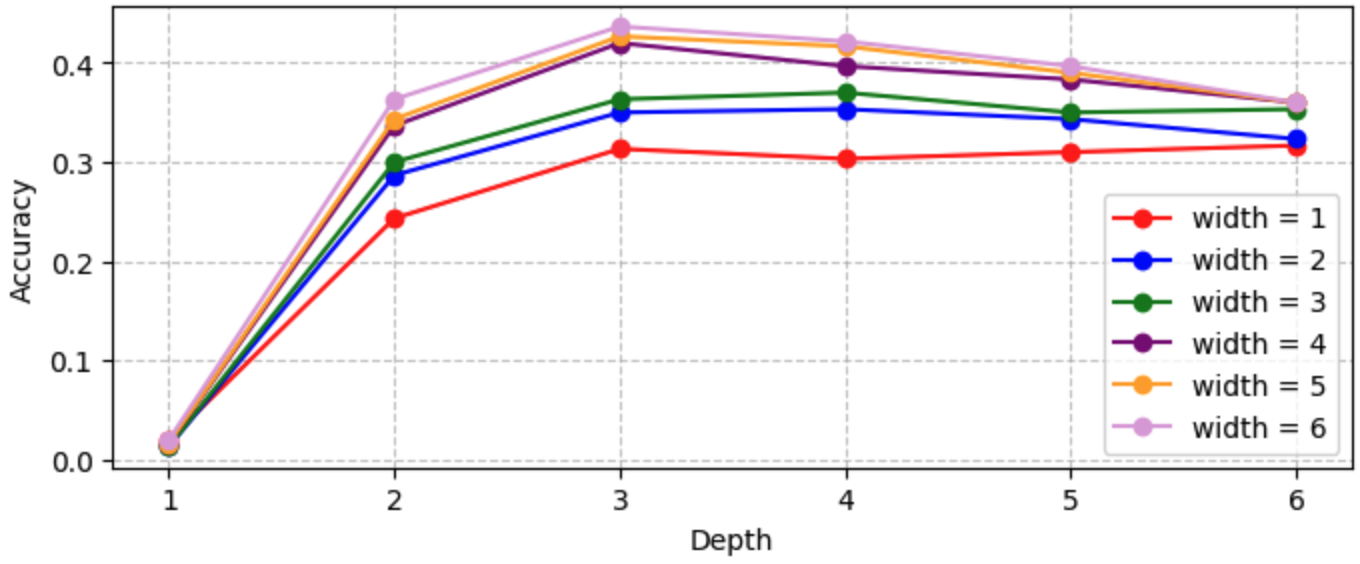}}
\vskip -0.1in
\caption{The accuracy plot with different choice of expansion depth and beam width.}
\label{fig: Search Space}
\end{center}
\vskip -0.2in
\end{figure}

\section{Related Work}

\subsection{Reasoning Ability Enhancement in LLMs}
Methods to improve LLMs' reasoning ability can be categorized into tuning-based and inference-time approaches.
Tuning-based methods focus on fine-tuning LLMs with high-quality rationales. STaR \cite{zelikman2022star} iteratively bootstraps rationales through generation and filtering. TRICE \cite{hoffman2024training} employs MCMC sampling to construct training data with rationales and leverages rationalization for failed cases. CPO \cite{zhang2024chain} extracts step-wise preference data from ToT's search tree for DPO-based optimization.
Inference-based methods design structured reasoning frameworks to guide LLMs during inference. Chain-of-thought (CoT) \cite{wei2022chain, kojima2022large} breaks down reasoning into sequential steps. Tree-of-thoughts (ToT) \cite{yao2024tree, long2023large} enables multi-path exploration with backtracking. Graph-of-Thoughts (GoT) \cite{besta2024graph} extends to arbitrary graph topologies for complex reasoning patterns.
Tree-based methods have emerged as mainstream by balancing exploration capability with structural simplicity.

\subsection{Linear Representation Hypothesis in LLMs}
The Linear Representation Hypothesis (LRH), initially proposed in word embeddings \cite{mikolov2013linguistic}, suggests that semantic features exist as linear directions in activation space. Recent work has extended this to LLMs \cite{luo2024pace, von2024language, zou2310representation, park2311linear}, showing that high-level concepts like truthfulness \cite{li2024inference, burns2022discovering}, morality \cite{zou2310representation}, and factual knowledge \cite{gurnee2023language} can be represented linearly in model's activation space.
This finding enables two key applications: detection and guidance. For detection, linear classifiers can effectively probe specific concepts \cite{chen2024inside, du2024haloscope}, with their high performance indicating the linear encoding of these concepts. For guidance, these identified directions can be leveraged to steer model behavior during inference \cite{lee2409programming, li2024inference, zhao2024steering}.




\section{Discussion and Limitation}
Despite our approach's effectiveness, we acknowledge several limitations that may warrant further investigation. 
First, our current implementation relies on fixed token lengths to segment intermediate thoughts during tree node expansion, which may disrupt natural reasoning by forcing arbitrary branching points. 
Future work should explore more flexible, semantic-aware splitting criteria to better preserve complete units of reasoning.

Second, while the answer pool achieves promising coverage rates for correct answers, our final answer selection process has room for improvement. 
The observable gap between coverage and accuracy suggests current chain evaluation and branch-aggregation strategies may not optimally capture answer quality. 
Future research could develop more sophisticated scoring metrics and aggregation methods.

Furthermore, while our classifier is trained using outcome-based supervision, it guides exploration by assessing progress at each step. This misalignment between training objectives (final outcomes) and guidance criteria (intermediate progress) may limit its effectiveness. 
Future work could explore step-wise feedback or hierarchical evaluation to better align the classifier training for better progress-based guidance.

\section{Conclusion}

In this work, we present ThoughtProbe, a pure inference-time framework that leverages LLMs' intrinsic reasoning capabilities through classifier-guided exploration. 
We first demonstrate that thoughtfulness - a proxy concept of intrinsic reasoning - is linearly separable in LLMs' activation space. 
Through comprehensive probing experiments, we find different LLM architectures encode thoughtfulness distinctly across representation types and layers, achieving strong classification performance using simple linear classifiers. 
Also, the trained classifiers logit consistently assign higher scores to correct, thoughtful responses over incorrect ones, regardless of response length.
Building on this discovery, we develop a classifier-guided beam search algorithm that effectively explores the reasoning space by prioritizing thoughtful candidate token sequences. 
Our framework combines tree-structured exploration with branch aggregation for optimal answer selection, enabling systematic discovery and utilization of valid reasoning chains within LLMs' response space.
Extensive experiments across multiple mathematical reasoning benchmarks demonstrate the effectiveness of our approach, achieving significant improvements over existing methods. 
These findings advance both our understanding of LLMs' intrinsic reasoning mechanisms and enable their practical application without human intervention.


\section{Impact Statement}
This paper presents work whose goal is to advance the field of Machine Learning. There are many potential societal consequences  of our work, none which we feel must be specifically highlighted here.

\nocite{langley00}

\bibliography{example_paper}
\bibliographystyle{icml2025}

\newpage
\appendix
\onecolumn
\section{Appendix}



\subsection{Complete proof for Theorem 3.1}

\textbf{Binary Classification Setting:}
Let $\mathcal{X}$ be the input space. For any $x \in \mathcal{X}$:
\begin{itemize}
\item $P(y = 1|x)$ denotes the probability of positive class
\item $l(x) := \text{logit}$ is the classifier logit output
\end{itemize}

\textbf{Bradley-Terry Preference Model:}
For pairs $(i, j) \in \mathcal{X} \times \mathcal{X}$:
\begin{itemize}
\item $P(i \succ j)$ denotes preference probability
\item $P(i \succ j) = \frac{\exp(r(i))}{\exp(r(i)) + \exp(r(j))}$
\item $r(\cdot)$ is the underlying reward function
\end{itemize}

\begin{lemma}[Classification-Preference Connection]
For any $x \in \mathcal{X}$:
\begin{equation}
P(y = 1|x) = \mathbb{E}_{j\sim p(j)}[P(x \succ j)]
\end{equation}
\end{lemma}

\begin{proof}
In binary classification, we can view the process as a competition where:
\begin{itemize}
\item $x$ competes against a random competitor $j$
\item $P(y = 1|x)$ represents the winning probability of $x$
\item When $j$ is randomly sampled from $p(j)$, this probability equals $\mathbb{E}_{j\sim p(j)}[P(x \succ j)]$
\end{itemize}
\end{proof}

\begin{theorem}[Classification-Reward Equivalence]
Let $l(x)$ be the classifier logit output and $r(x)$ be the reward function. There exists a constant $C$ such that:
\begin{equation*}
l(x) \geq r(x) - C
\end{equation*}
\end{theorem}

\begin{proof}
Under the Bradley-Terry model:
\begin{equation}
P(y = 1|x) = \mathbb{E}_j\left[\frac{\exp(r(x))}{\exp(r(x)) + \exp(r(j))}\right]
\end{equation}

By Jensen's inequality, since $f(t) = \frac{a}{a + t}$ is convex in $t$ for $a > 0$:
\begin{equation}
P(y = 1|x) \geq \frac{\exp(r(x))}{\exp(r(x)) + \mathbb{E}[\exp(r(j))]}
\end{equation}

Taking logit transformation:
\begin{align*}
l(x) &= \text{logit} P(y = 1|x) = \log\frac{P(y = 1|x)}{1 - P(y = 1|x)} \geq r(x) - \underbrace{\log(\mathbb{E}[\exp(r(j))])}_{C}
\end{align*}
\end{proof}

\begin{corollary}[Order Preservation]
For any two instances $x_1$ and $x_2$:
\begin{equation}
r(x_1) > r(x_2) \iff l(x_1) > l(x_2)
\end{equation}
\end{corollary}

\textbf{Key Implications:}
\begin{itemize}
\item Logit outputs preserve preference ordering
\item The difference is only a constant shift
\item Binary classification effectively implements preference learning
\item Classifier logits provide valid reward signals
\end{itemize}

\section{Additional Experimental Details and Results}

\subsection{Classifier Training Data Construction Details}
\label{app: classifier training data}
We employ Top-K-Start Greedy Decoding, which explores alternative top-k tokens at the first decoding step, followed by greedy decoding for subsequent steps, to sample 10 distinct responses.
To ensure response quality, we first filter the sampled responses to remove potential repetitions of input questions. 
LLMs occasionally exhibit a pattern where they restate the original question after providing their answer.
We apply a post-processing step to extract only the solution-relevant content, ensuring each response contains purely reasoning and answer components without redundant question restatements.
After filtering, we label hidden representations from all network layers: representations corresponding to responses that are both correct and longer than 30 tokens serve as positive samples for thoughtfulness, while those from incorrect responses shorter than 30 tokens serve as negative samples. This labeling propagates through all hidden layers, creating a comprehensive training dataset at the representation level.

\subsection{Classifier Classification Performance}
\label{app:classifier analysis}
Building upon our main findings, we conduct an extensive evaluation using both Logistic Regression (LR) and Support Vector Machine (SVM) classifiers, assessed through Accuracy, F1-score, and AUC-ROC metrics. 
As shown in Figures \ref{app:LR} and \ref{app:SVM}, these complementary metrics reinforce and extend our key observations:

(1) Representation type analysis:
The distinct patterns observed across different LLMs are consistently reflected across all metrics. In Mistral-7b, hidden states maintain their superior performance across both classifiers and all evaluation metrics, with MLP and attention activations showing comparable but slightly lower performance. For Phi-1.5, the notable fluctuation in hidden states and the stable superiority of attention outputs are robustly captured by all metrics. In Gemma2-2b, attention activations consistently demonstrate the strongest discriminative power across all evaluation criteria, while hidden states and MLP outputs show substantial variations. This consistent pattern across different evaluation frameworks strengthens our observation about model-specific architectural strategies for encoding thoughtful reasoning.

(2) Layer-wise analysis:
The layer-specific trends identified in our main analysis persist across different classification approaches and metrics. Mistral-7b's progressive improvement in deeper layers is consistently observed in both LR and SVM results, regardless of the evaluation metric used. 
The more distributed patterns in Phi-1.5 and Gemma2-2b, characterized by fluctuations without clear directional trends, are similarly preserved across all evaluation frameworks. 
These consistent findings across multiple metrics provide strong validation for our conclusions about how different LLMs architecturally encode thoughtfulness features.

\subsection{Classifier logit value analysis}
\label{app:score analysis}
We conduct a detailed analysis of classifier logit value distributions across multiple language models (Mistral-7b, Gemma2-2b, Phi-1.5). 
Using both Logistic Regression and SVM classifiers, we compare the distributional patterns between the most discriminative layers (top-3 F1 scores) and least discriminative layers (bottom-3 F1 scores).

Figures \ref{app:score_gemma_short}, \ref{app:score_gemma_long} present comprehensive comparisons of classifier logit values for Gemma2-2B. Figures \ref{app:score_mist_short}, \ref{app:score_mist_long} present comprehensive comparisons of classifier logit values for Mistral-7B. Figures \ref{app:score_phi_short}, \ref{app:score_phi_long} present comprehensive comparisons of classifier logit values for Phi-1.5. 

Specifically, Figure \ref{app:score_gemma_short} compares thoughtful correct responses against intuitive responses, while Figure \ref{app:score_gemma_long} contrasts thoughtful correct responses with thoughtful incorrect responses. 
These comparisons are conducted within both top-3 and bottom-3 performing layers (ranked by F1-scores), spanning across different classifier architectures (Logistic Regression and SVM) and various representation types (attention activations, MLP activations, and hidden states), providing a thorough validation of the scoring
and ranking ability of the classifier's logit.

Notably, in attention activations, which achieve the best classification performance among all representation types, thoughtful correct responses consistently receive higher logit values than both intuitive responses and thoughtful but incorrect responses, demonstrating the robust discriminative ability of our approach. However, this clear ranking pattern is occasionally violated in MLP activations and hidden states, where thoughtful correct responses sometimes receive lower logit values than the other response types. Moreover, this ranking trend is more pronounced in top-3 performing layers compared to bottom-3 layers, suggesting that layers with stronger discriminative power better preserve the desired response quality ordering.

Similar patterns are observed in Mistral-7B and Phi-1.5 models, indicating that our trained classifiers demonstrate strong scoring and ranking capabilities across different model architectures, making them reliable guides for thought space exploration.

\begin{figure}[ht]
\vskip 0.2in
\begin{center}
\centerline{\includegraphics[scale = 0.5]{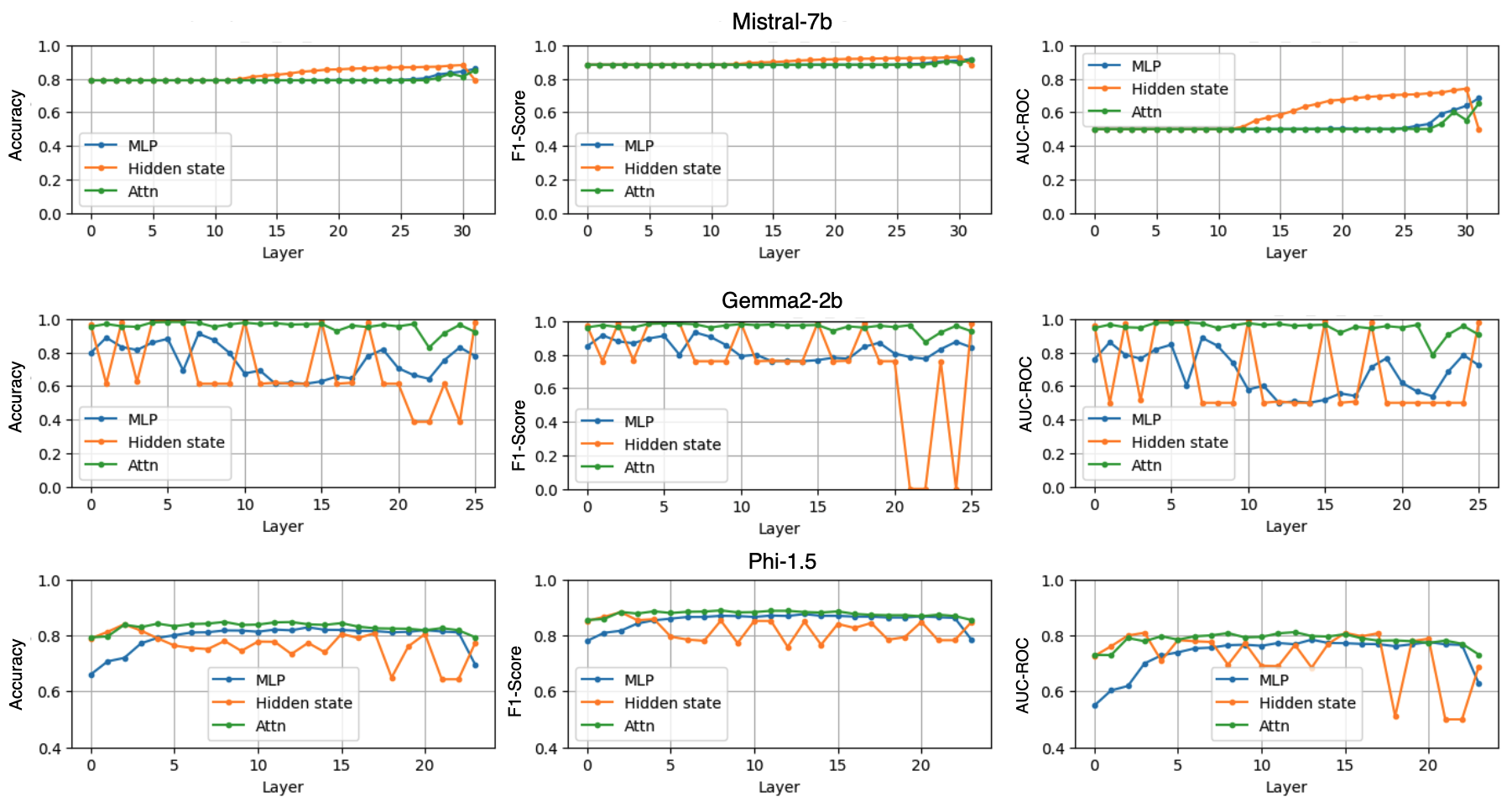}}
\vskip -0.1in
\caption{LR classification performance across LLMs and representation types}
\label{app:LR}
\end{center}
\vskip -0.2in
\end{figure}

\begin{figure}[ht]
\begin{center}
\centerline{\includegraphics[scale = 0.5]{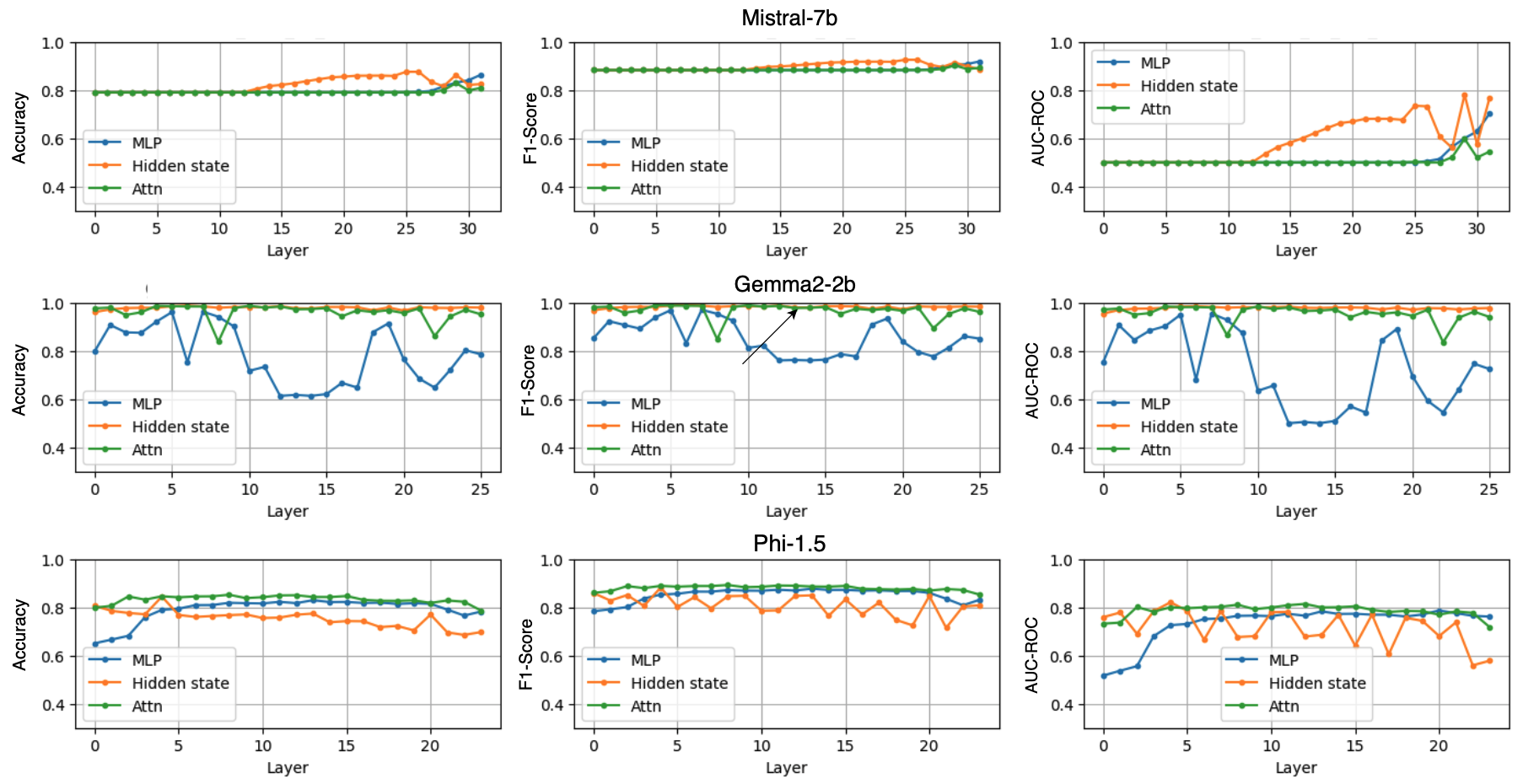}}
\vskip -0.1in
\caption{SVM classification performance across LLMs and representation types}
\label{app:SVM}
\end{center}
\vskip -0.2in
\end{figure}

\begin{figure}[ht]
\begin{center}
\centerline{\includegraphics[scale = 0.7]{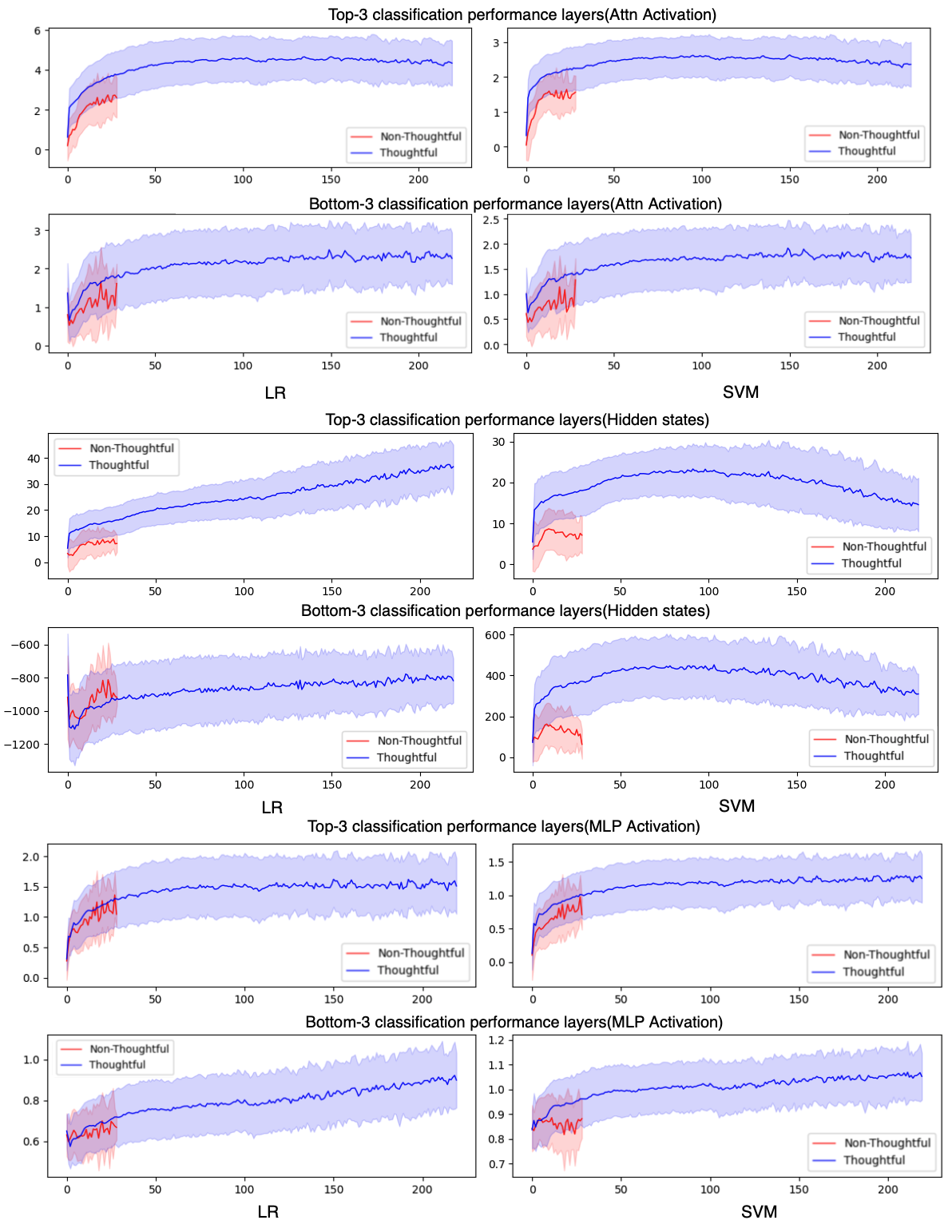}}
\vskip -0.1in
\caption{Mean logit values and variance regions in Gemma2-2b, comparing lengthy thoughtful correct responses with concise intuitive incorrect ones.}
\label{app:score_gemma_short}
\end{center}
\vskip -0.2in
\end{figure}

\begin{figure}[ht]
\begin{center}
\centerline{\includegraphics[scale = 0.67]{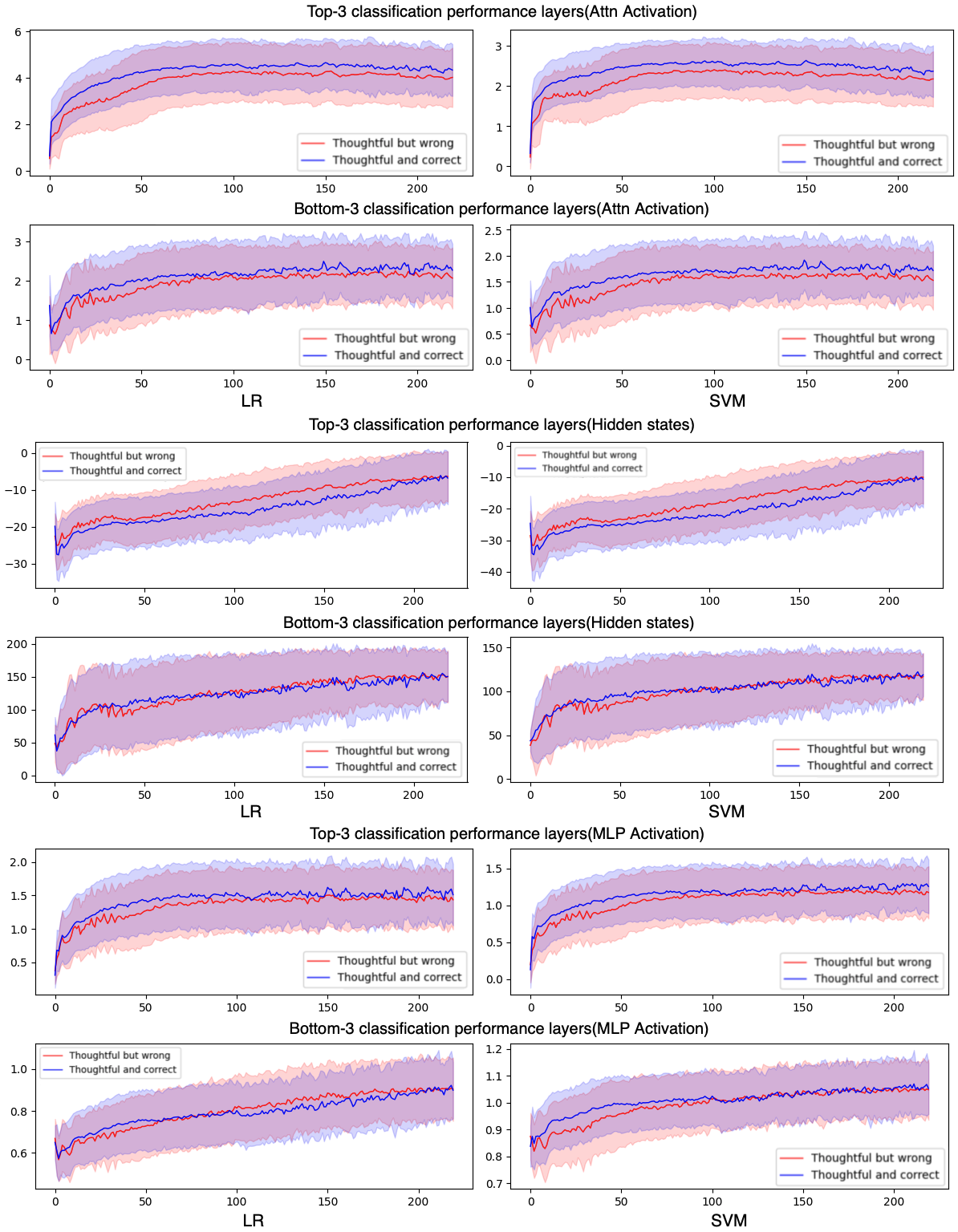}}
\vskip -0.1in
\caption{Mean logit values and variance regions in Gemma2-2b, comparing lengthy thoughtful correct responses with lengthy incorrect ones.}
\label{app:score_gemma_long}
\end{center}
\vskip -0.2in
\end{figure}

\begin{figure}[ht]
\begin{center}
\centerline{\includegraphics[scale = 0.72]{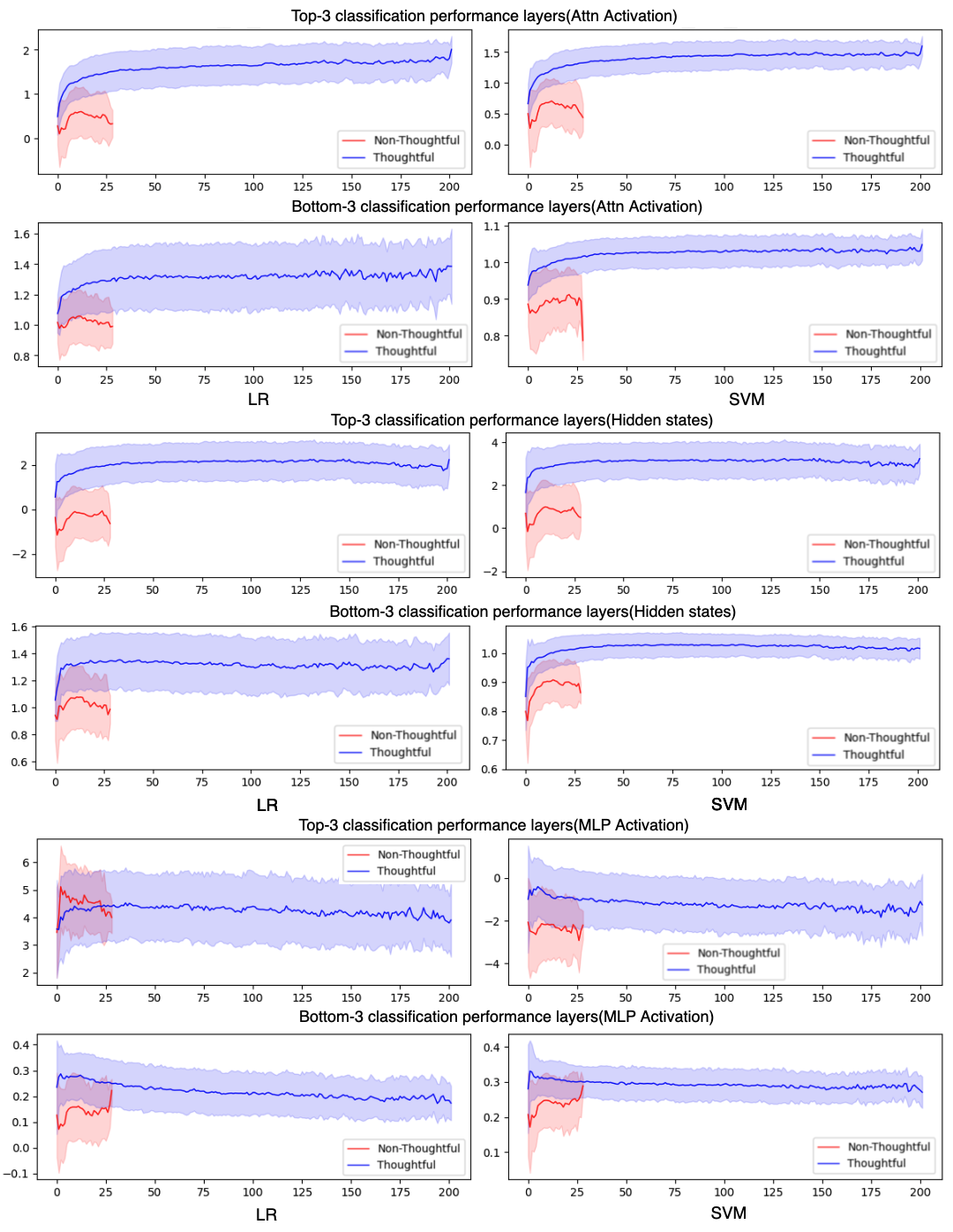}}
\vskip -0.1in
\caption{Mean logit values and variance regions in Mistral-7b, comparing lengthy thoughtful correct responses with concise intuitive incorrect ones.}
\label{app:score_mist_short}
\end{center}
\vskip -0.2in
\end{figure}

\begin{figure}[ht]
\begin{center}
\centerline{\includegraphics[scale = 0.67]{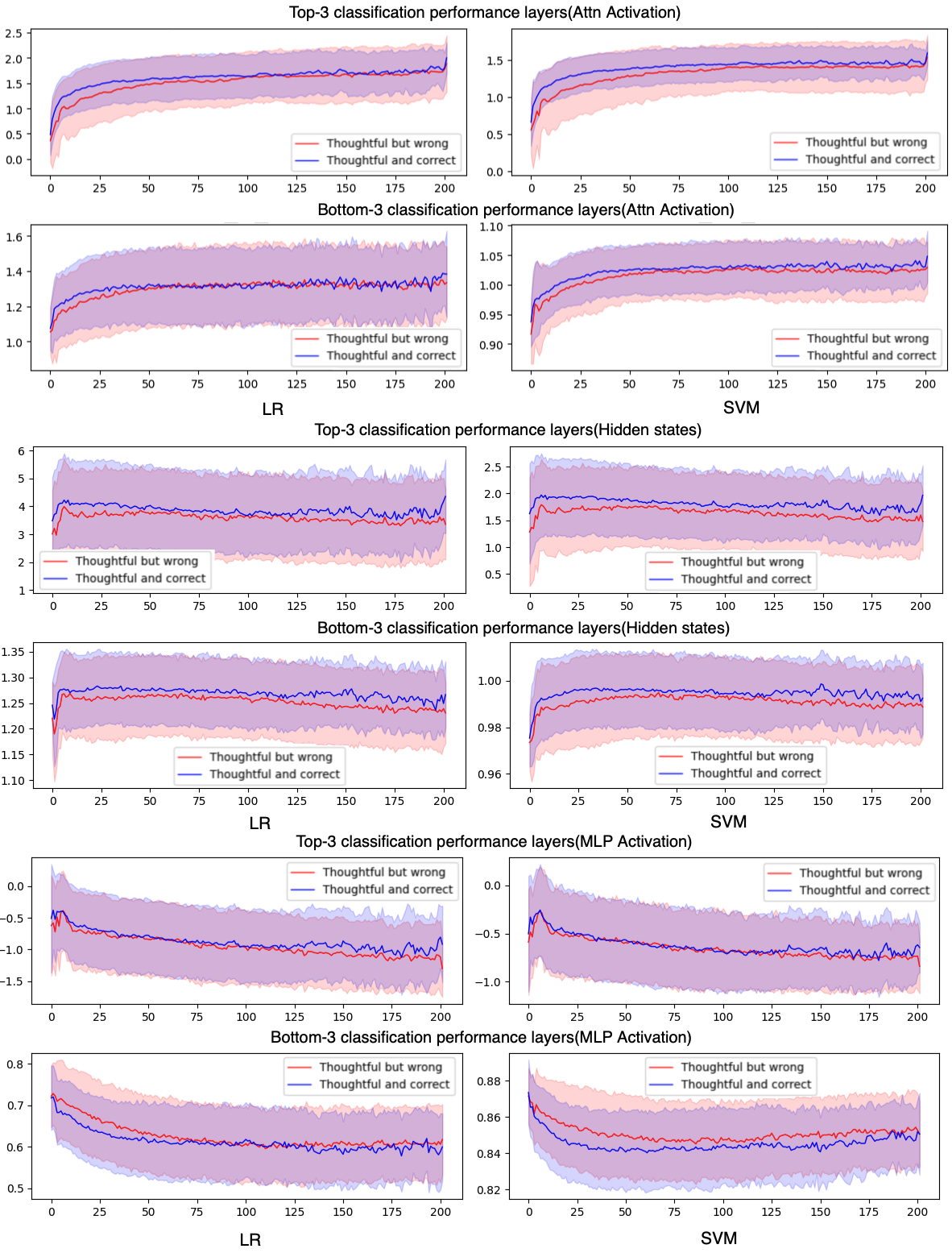}}
\vskip -0.1in
\caption{Mean logit values and variance regions in Mistral-7b, comparing lengthy thoughtful correct responses with lengthy incorrect ones.}
\label{app:score_mist_long}
\end{center}
\vskip -0.2in
\end{figure}

\begin{figure}[ht]
\begin{center}
\centerline{\includegraphics[scale = 0.6]{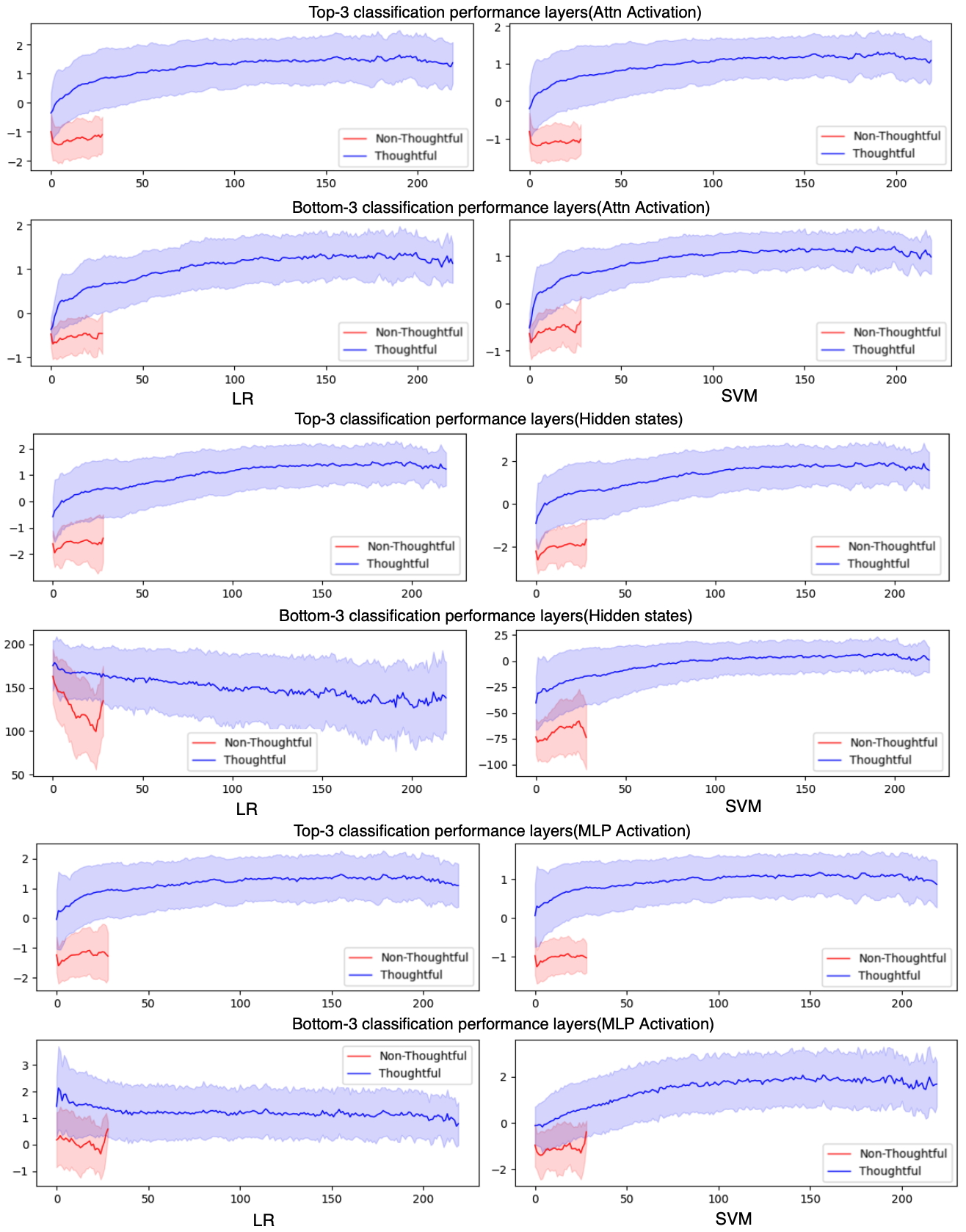}}
\vskip -0.1in
\caption{Mean logit values and variance regions in Phi-1.5, comparing lengthy thoughtful correct responses with concise intuitive incorrect ones.}
\label{app:score_phi_short}
\end{center}
\vskip -0.2in
\end{figure}

\begin{figure}[ht]
\begin{center}
\centerline{\includegraphics[scale = 0.67]{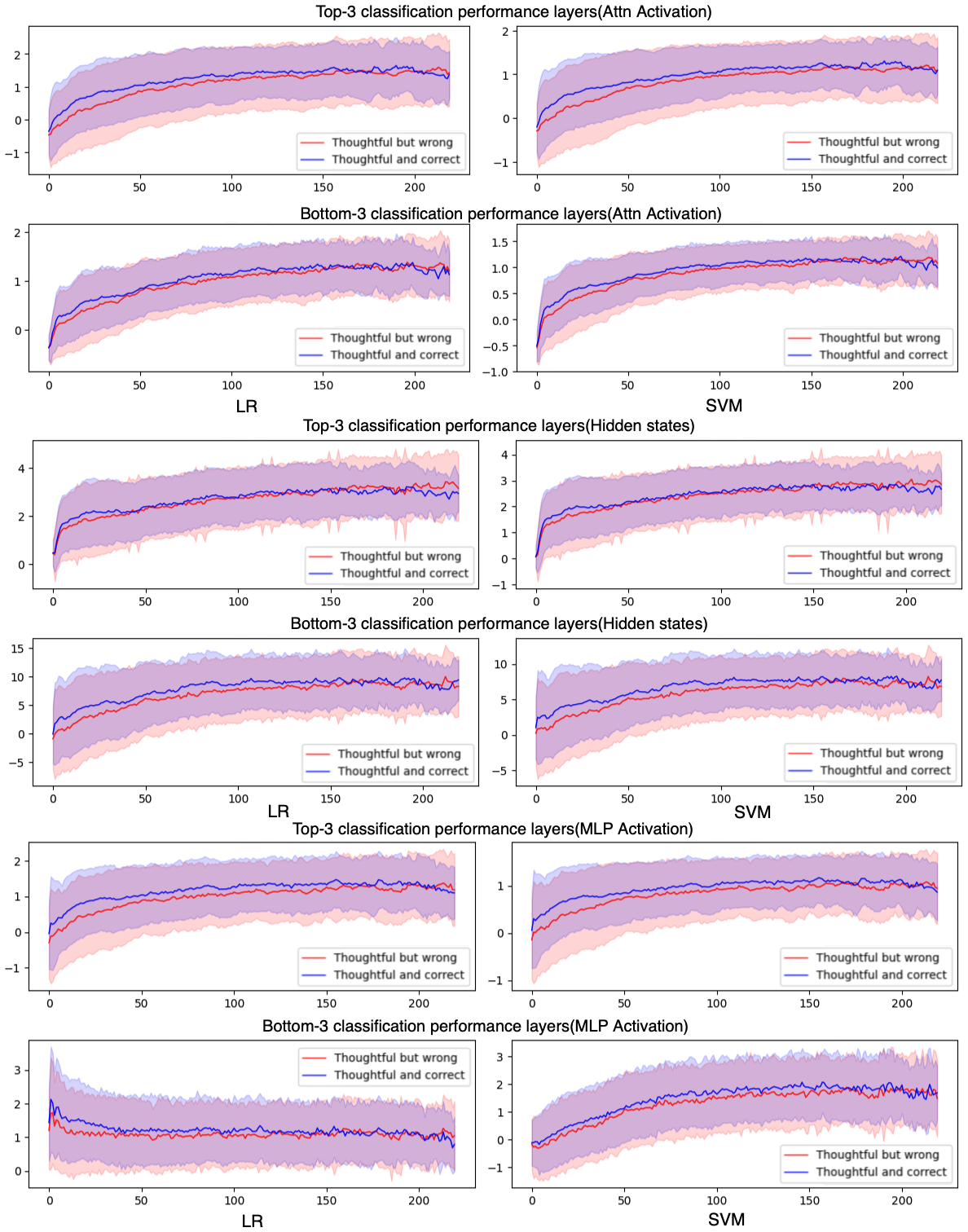}}
\vskip -0.1in
\caption{Mean logit values and variance regions in Phi-1.5, comparing lengthy thoughtful correct responses with lengthy incorrect ones.}
\label{app:score_phi_long}
\end{center}
\vskip -0.2in
\end{figure}

\end{document}